%% file: conv-node2.tex
\documentclass[11pt]{amsart}
\input{preamble.tex}

\newtheorem{theorem}{Theorem}[section]
\newtheorem{definition}{Definition}[section]
\newtheorem{proposition}{Proposition}[section]
\newtheorem{lemma}{Lemma}[section]
\theoremstyle{remark}
\newtheorem{remark}{Remark}[section]
\numberwithin{equation}{section}

\usepackage{enumitem}

\author{Qianxiao Li}
\author{Ting Lin}
\author{Zuowei Shen}

\begin{document}

\title[UAP of Fully CNN]{On the Universal Approximation Property of Deep Fully Convolutional Neural Networks}

\begin{abstract}
    We study the approximation of shift-invariant or equivariant functions by deep fully convolutional networks from the dynamical systems perspective. We prove that deep residual fully convolutional networks and their continuous-layer counterpart can achieve universal approximation of these symmetric functions at constant channel width. Moreover, we show that the same can be achieved by non-residual variants with at least 2 channels in each layer and convolutional kernel size of at least 2. In addition, we show that these requirements are necessary, in the sense that networks with fewer channels or smaller kernels fail to be universal approximators.
    \end{abstract}
\maketitle


\section{Introduction}

Convolutional Neural Networks (CNN) are widely used as fundamental building
blocks in the design of modern deep learning architectures, for it can extract key data features with much fewer parameters, lowering both memory requirement and computational cost. When the input data contains spatial structure, such as pictures or videos, this parsimony often does not hurt their performance.
This is particularly interesting in the case of fully convolutional neural networks (FCNN)
\cite{long2015fully}, built by the composition of convolution, nonlinear activation and summing (averaging) layers,
with the last layer being a permutation invariant pooling operator, see Figure~\ref{fig:cnn}.
\begin{figure}[htbp]
    \centering
    \includegraphics*[width=\textwidth]{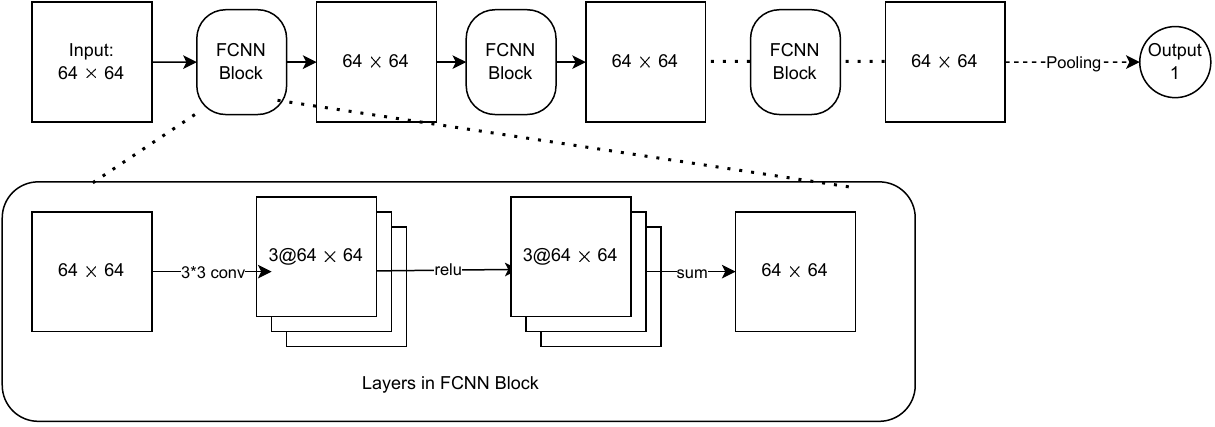}
    \caption{An illustration of fully convolutional neural network.}
    \label{fig:cnn}
    \end{figure}
Consequently, a prominent feature of FCNN is that, when shifting
the input data indices (e.g. picture, video, or other higher-dimensional spatial data), the output result should remain the same. This is called {\it shift invariance.}
An example application of FCNN is image classification problems where
the class label (or class assignment probability, under the softmax activation)
of the image remains the same under translating the image
(i.e. shifting the image pixels).
A variant of FCNN applies to problems where the output data
has the same size as the input data,
e.g. pixel-wise segmentation of images \cite{badrinarayanan2017segnet}.
In this case, simply stacking the fully convolutional layers is enough.
We call this type of networks equivariant fully convolutional neural network (eq-FCNN),
since when shifting the input data indices,
the output data indices shift by the same amount.
This is called {\it shift equivariance}.
It is believed that the success of these convolutional architectures
hinges on shift invariance or equivariance, which capture
intrinsic structures in spatial data.
From an approximation theory viewpoint, this presents a delicate trade-off between
expressiveness and invariance: layers cannot be too complex to break the invariance property,
but should not be too simple that it loses approximation power.
The interaction of invariance and network architectures has been a subject of
intense study in recent years.
For example, \cite{cohen2016steerable} designed steerable CNNs to handle the motion group for robotics.
Deep sets \cite{NIPS2017_f22e4747} are proposed to accommodate general permutation invariance and equivariance.
Other approaches to build equivariance and shift invariance include
parameter sharing \cite{ravanbakhsh2017equivariance,han2022universal} and the homogeneous space approach \cite{cohenwelling16,cohen2019general}.
See \cite{bronstein2017geometric} for a more recent survey.
Among these architectures, the FCNN is perhaps the simplest and most widely used model.
Therefore, the study of its theoretical properties
is naturally a first and fundamental step for investigating other more
complicated architectures.

In this paper, we focus on the expressive power of the FCNN. Mathematically, we
consider whether a function $F$ can be approximated via the FCNN (or eq-FCNN) function family
in $L^p$ sense. This is also known as {\it universal approximation} in $L^p$.  In the
literature, many results on fully connected neural networks can be found, e.g. \cite{lu2021deep,lu2017expressive,shen2019deep,yarotsky2018optimal,zhang2023enhancing}.  However, relatively few
results address the approximation of shift invariant functions via fully
convolutional networks.  An intuitive reason is that the symmetry
constraint (shift invariance) will hinder the unconditioned universal
approximation.  This can be also proved rigorously.
In \cite{li2022}, it is shown that if a function can be approximated by
an invariant function family to arbitrary accuracy,
then the function itself must be invariant.
As a consequence, when we consider the approximation property of the FCNN, we should
only consider shift invariant functions. This brings new difficulty for
obtaining results compared to those for fully connected neural networks.
For this reason, many existing results on convolutional network approximation rely on
some ways of breaking shift invariance, thus applying to general function
classes without symmetry constraints \cite{oono2019approximation}.
Moreover, current results on convolutional
networks usually require (at least one) layers to have a large number of
channels \cite{hwang2022universal}.

In contrast, we establish universal approximation results for fully convolutional networks where shift
invariance is preserved.  Moreover, we show that approximation can be achieved
by increasing depth at constant channel numbers, with fixed kernel size in each layer.
The main result of this paper (Theorem~\ref{thm:main}) shows that if we choose $\relu$ as the
activation function and the terminal layer is chosen as a general pooling
operator satisfying mild technical conditions (e.g. max, summation), then convolutional
layers with at least 2 channels and kernel size at least 2 can achieve universal
approximation of shift invariant functions via repeated stacking (composition).
The result is sharp in the sense that neither the size of convolution kernel nor
the channel number can be further reduced while preserving the universal
approximation property.

To prove the result on FCNN, we rely on the dynamical systems approach where
residual neural networks are idealized as continuous-time dynamical systems.
This approach was introduced in~\cite{weinan2017proposal} and first used to
develop stable architectures~\cite{Haber2017} and control-based training
algorithms~\cite{li2017maximum}.  This is also popularized in the machine
learning literature as neural ODEs \cite{Chen2018}.  On the approximation theory front, the
dynamical systems approach was used to prove universal approximation of general
model architectures through composition~\cite{li2019deep}.  The work of~\cite{li2022},
extended the result to functions/networks with symmetry constraints,
and as a corollary obtained a universal approximation result
for residual fully convolutional networks with kernel sizes equal to the image size.
The results in this paper restrict the size of kernel in a more practical way,
and can handle common architectures for applications, which typically use kernel sizes ranging from
$3-7$.
Moreover, we also establish here the sharpness of the requirements on channel numbers and kernel sizes.  
The restriction on width and kernel size actually can provide more interesting results in the theoretical setting. 
This is because if we establish our approximation results using finite (and minimal) width and kernel size requirements, they can be used to obtain the universal approximation property for a variety of larger models by simply showing them to contain our minimal construction. 

In summary, the main results of this work are as follows.
First, we prove the universal approximation property of shift-invariant functions
for both continuous and time-discretized deep residual fully convolutional neural networks
having kernel size of at least 2. This result concerns deep but possibly narrow
residual neural networks. We provide a sufficient condition on
the universal approximation property with respect to shift
invariance, which allows one to check the universality of any
given deep residual architecture.
In particular, the result rely neither on a specific choice of nonlinear
activation function, nor a choice of the last layer.
Further, we prove the universal approximation property of fully
convolutional neural network with ReLU activations having no less than two
channels each layer, and kernel size of at least 2.
Finally, we show that the channel number and kernel size requirements
above are sharp, in that networks with fewer channels or kernel sizes
do not possess the universal approximation property. 
The above three points hold true also for the approximation of shift equivariant mappings via eq-FCNN.

\section{Formulation and main results}
\label{sec:main}

In this section, we introduce the notation and formulation of the approximation
problem, and then present our main results.
We first recall the definition of convolution: 
Consider two rank $d$ tensors $\bm x$ and $\bm y \in \mathbb{X}$,
where $\mathbb{X} := \mathbb{R}^{n_1} \times \dots \times \mathbb{R}^{n_d}$.
We denote by $\bm n = [n_1, \dots, n_d]$ the data dimensions.
Define the \textit{convolution} of $\bm x, \bm y$ by
$\bm z = \bm x \ast \bm y$ with
$$
    [\bm z]_{\bm i} = \sum_{\bm j} [\bm x]_{\bm j} [\bm
    y]_{\bm i + \bm j - \bm 1}.
$$
Here, $\bm i, \bm j$ are multi-indices (beginning with $1$) and
the arithmetic uses the periodic boundary condition.
Taking $\bm x \in \mathbb R^{3\times 3}$ as an example, we denote
\begin{equation}
\bm x = \begin{bmatrix} [\bm x]_{(1,1)} & [\bm x]_{(1,2)} &   [\bm x]_{(1,3)} \\ [\bm x]_{(2,1)} & [\bm x]_{(2,2)} &   [\bm x]_{(2,3)} \\ [\bm x]_{(3,1)} & [\bm x]_{(3,2)} &   [\bm x]_{(3,3)}\end{bmatrix},
\end{equation}
where $[\bm x]_{(1,4)}$ is identified with $[\bm x]_{(1,1)}$, and similarly for the other indices.

Let us also define the \textit{translation operator} $\trans_{\bm k}$ with respect
to a multi-index $\bm k$ by $[\trans_{\bm k} \bm x]_{\bm i} = [\bm x]_{\bm i + \bm k}$.
The key symmetry condition concerned in this paper - shift equivariance - can now
be stated as the following commuting relationship: 
$$
        \trans_{\bm k} (\bm x \ast \bm y) = \bm x \ast (\trans_{\bm k} \bm y),\quad \forall x, y \in \mathbb X, \,\forall \bm k.
$$

We now introduce the definition of the fully convolutional neural network (FCNN)
architecture we subsequently study.
Let
\begin{equation}
    \label{eq:Fcal}
    \Fcal_{r}:=\left\{\sum_{i = 1}^{r} v_i \sigma(\bm w_i \ast \cdot + b_i \bm
    1), \bm w_i \in \mathbb X, v_i , b_i \in \mathbb R\right\}
\end{equation}
be a function family representing possible forms for each convolutional layer with $r$ channels.
Here, $\sigma(x) = \max(x,0)$ is the ReLU function.

Let the final layer be a \textit{pooling operation} $g: \mathbb X \to \mathbb R$ obeying the following condition:
 $g$ is Lipschitz, and permutation invariant with respect to all the coordinates of its input data,
i.e., the value of $g$ does not depend on the order of its inputs.
Examples of such a pooling operator include \textit{summation} 
$ g(\bm x) = \bm x \mapsto \sum_{\bm i} [\bm x]_{\bm i}$ and \textit{max} $g(\bm x) = \bm x \mapsto \max_{\bm i} [\bm x]_{\bm i}.$

\begin{remark}
Note that the assumption is stronger than just requiring $g$ to be shift invariant.
\end{remark}

In the above definition \eqref{eq:Fcal}, the convolution kernel has the same size of the input data. 
In practice, however, the convolutional kernel used will be more restrictive, say a
kernel size of $3$ or $5$.  To study the effect of kernel size, we define the
support for an element $\bm x \in \mathbb X$ as $\supp(\bm x) :=
(j_1,j_2,\cdots,j_d)$, where $j_s$ is the minimal number such that if
the multi-index $\bm i$ has $i_s > j_s$ for some $s$, then $[\bm x]_{\bm i} = 0.$ For example, the support of tensor $\bm x = \begin{bmatrix} 0 & 1 & 0 \\ 1 & 0 & 0 \\ 1 & 0 & 0 \end{bmatrix}$ is $(3,2)$.

\begin{remark}
Two remarks on this definition of support are in order.
\begin{itemize}
    \item[-]
First, the element $\bm x
\in \mathbb X$ with $\supp (\bm x) \le \bm j = (j_1,j_2,\cdots, j_d)$ can be
identified with an element $\tilde{\bm x} \in \mathbb R^{j_1\times
j_2\times\cdots\times j_d}$. \footnote{In what follows, we define for multi-indices $\bm i, \bm j$ the partial order
$\bm i \ge \bm j$ if $i_s \ge j_s, s = 1,2,\cdots.$.}
\item[-] 
Second, a convolution kernel with size of $s$ can
be regarded as a tensor $\bm w \in \mathbb X$ with support $\le \bm s =
(s,s,s,\cdots, s)$.
\end{itemize}
\end{remark}

Thus, we may define the convolutional layer family with support up to $\bm \ell$ as
$$
    \Fcal_{r,\bm \ell} := \left\{\sum_{i = 1}^{r} v_i \sigma(\bm w_i \ast \cdot +
    b_i \bm 1), \bm w_i \in \mathbb X, \supp(\bm w_i) \le \bm \ell, v_i , b_i \in
    \mathbb R\right\}.
$$
With these notations in mind, we now introduce the following hypothesis spaces
defining fully convolutional neural networks and their residual variants
\begin{align}
    \cnn_{r,\bm \ell}
    &=
    \{
        g \circ \bm f_m\circ \cdots \circ\bm f_1 : \bm f_1,\cdots, \bm f_m \in \Fcal_{r, \bm \ell}
        ,
        m \geq 1
    \},\\
    \rescnn_{r,\bm \ell}
    &=
    \{
        g \circ (\id + \bm f_m)\circ \cdots \circ(\id + \bm f_1) : \bm f_1,\cdots,
        \bm f_m \in \Fcal_{r, \bm \ell},
        m \geq 1
    \}.
\end{align}

Observe that all functions in the families
$\cnn_{\cdot,\cdot}$ and $\rescnn_{\cdot,\cdot}$ are shift invariant in the following sense.
\begin{definition}[Shift Invariance]
A function $\varphi : \mathbb{X} \rightarrow \R$
is called shift invariant if
$
    \varphi(\bm x) = \varphi(\trans_{\bm k} \bm x)$ for all $\bm x \in \mathbb X, \bm k
    \le \bm n.
$
A function family $\mathcal X$ is called {\it shift invariant} if for all its member are shift invariant.
\end{definition}

\begin{definition}[Shift Invariant UAP]
A function family $\mathcal X$ satisfies the shift invariant universal approximation property
(shift invariant UAP for short) if
\begin{enumerate}
\item
The function family $\mathcal X$ is shift invariant, and \item For any shift invariant continuous (or $L^p$) function $\psi$, tolerance $\varepsilon >0 $, compact set $K \subset \mathbb X$ and $p \in [1,\infty)$, there exists $\varphi \in \mathcal X$ such that
$ \|\psi - \varphi\|_{L^p(K)} \le \varepsilon.$
\end{enumerate}
\end{definition}



For any family $\mathbf{F}$ of functions $\mathbb{X}\to\mathbb{R}$,
let us define
$\mathbf{F} +\mathbb{R}:=\{\varphi + b, \varphi \in \mathbf{F} , b\in \mathbb R\}$.
This expands the hypothesis space by adding
a constant bias to the original function family $\mathbf{F}$.
The main result of this paper is as follows. \footnote{In this paper, we always fix a $p \in [1,\infty).$} 
\begin{theorem}[Universal Approximation Property of $\cnn$]
\label{thm:main}
The following statements hold:
\begin{enumerate}[label=(\arabic*),ref=\ref{thm:main}.(\arabic*)]
    \item \label{thm:main1}
    The residual FCNN hypothesis space $\rescnn_{r, \bm \ell}$ possesses the
    shift invariant UAP for $r \ge 1$ and $\bm \ell \ge 2$.  The non-residual
    hypothesis space $\cnn_{r, \bm \ell}$ possesses the shift invariant UAP for
    $r \ge 2$ and $\bm \ell \ge 2$.
    \item \label{thm:main2}
    The kernel size $\bm 2$ is optimal in the following sense: for $\bm \ell$
    with $\min \ell_s = 1$, then neither $\rescnn_{\infty, \bm \ell} + \mathbb
    R$ nor $\cnn_{\infty, \bm \ell}  + \mathbb R$ possess the shift invariant
    UAP.
    \item \label{thm:main3}
    The channel-width requirement for non-residual fully convolutional neural
    network is optimal, in the sense that the function family  $\cnn_{1, \bm
    \infty} + \mathbb R$ does not possess the shift invariant UAP.
    
\end{enumerate}

\end{theorem}

Notice that due the extended hypothesis space from the added bias,
the sharpness results are stronger than just implying that $\cnn_{\infty, \bm \ell}$
or $\rescnn_{\infty, \bm \ell}$ does not possess the shift invariant UAP.
The reason we establish the sharpness results for $\cdot + \mathbb R$
is to ensure that the lack of approximation power does not arise from the
fact that the ReLU activation function $\sigma$ has non-negative range.
Note that this sign restriction does not affect the positive result,
since with at least 2 channels one can produce output ranges of any sign.
Although this theorem only considers the approximation of shift invariant architectures,
similar result can be established for the shift equivariant architectures.
We will discuss it in detail in \Cref{sec:eq}. 
Furthermore, in this section we restrict the activation function $\sigma$ to be the ReLU function,
but this restriction is necessary only for the non-residual case.
As we will see in \Cref{sec:proof},
for residual FCNNs we can relax our requirement on $\sigma$
to include a large variety of common activation functions.

Theorem~\ref{thm:main} indicates the following basic trade-off in the design of deep
convolutional neural network architecture: if we enlarge the depth of the neural
network, then even if we choose in each layer a simple function (in this
theorem, $2$ channels with each kernel in channel with size of $2$), we can
still expect a high expressive power. 
However, the mapping adopted in each
layer cannot be degenerate, otherwise it will fail to capture information of the input
data. The second and third part of this theorem tells that this degeneracy may
come from either channel number or the kernel size (support of the convolutional
kernel).

\subsection{Comparison with previous work}
We compare this theorem to existing works on the approximation theory of convolutional
networks and related architectures.
The existing result around the approximation 
capabilities of convolutional neural networks can be categorized into several classes. One either 
\begin{itemize}
\item[-] takes the kernel as full-size (same size as the input) (e.g.~\cite{li2022}) which is not often used in practice;

\item[-] assumes a sufficiently large channel number in order to adopt some kernel learning methods (e.g.~\cite{bietti2021approximation, favero2021locality, xiao2022eigenspace}) or averaging methods (e.g. \cite{yarotsky2018universal,bao2019approx,petersen2020equivalence,jiang_approximation_2021}.);

\item[-] removes the nonlinear activation function and reduces it to a linear approximation problem (e.g.~ \cite{zhou2020universality}), then uses complex fully connected terminal layer(s) to achieve approximation.

\end{itemize}

As a consequence, few, if any, results are obtained when the kernel size is small (and the channel number is fixed). Indeed, none of the results we are aware of have considered situations where both kernel size and the width are limited. However, this is in fact the case when designing deep (residual) NNs, as the ResNet family, where the primary change is increasing depth.
Our result indicates that even though each layer is relatively simple, much more complicated functions can be ultimately approximated via composition. Furthermore, our analytical techniques (especially for the residual case)
does not depend on the explicit form of the activation function and the
pooling operator in the last layer.

Another highlight feature of our result is with respect to the shift invariance, which might be overlooked in some approximation result for convolutional neural networks. We restrict our attention
to the periodic boundary condition case,
which leads to architectures that are exactly
shift invariant or equivariant.
This significantly confines the
expression power of the hypothesis spaces. If such symmetry is not imposed
on each layer, then one can achieve universal approximation of general functions,
but at the cost of breaking shift equivariance.
For example,
\cite{oono2019approximation} and \cite{okumoto2021learnability} drop the equivariant constraints and builds the deep
convolutional neural network with zero boundary condition, achieving
universal approximation property of non-symmetric functions.
This is because the boundary condition will deteriorate the
interior equivariance structure when the network is deep enough. Also, the shift invariance considered here is about the pixel (i.e. the input data), while some other attempts like
\cite{yang2020approximation} build a wavelet-like architecture to approximate a
function invariant to the spatial translation, i.e., functions satisfy that $f
= f(\cdot - \bm k)$ for $\bm k \in \mathbb Z$.

\subsection{Universal approximation property for equivariant neural networks}
    \label{sec:eq}

    If we remove the final layer in $\cnn$ or $\rescnn$,
    then we obtain a neural network whose output data is the same size as the input data.
    This is the original definition of FCNN introduced in \cite{long2015fully}, primarily
    used for pixel-wise image tasks.
    Correspondingly, the symmetry property is changed to shift equivariance,
    instead of shift invariance.
    This leads to the definition of the following hypothesis spaces
    that parallels the shift invariant counterparts.
    Define
    \begin{equation}
    \eqcnn_{r,\bm \ell} =\{ \bm f_{m} \circ \cdots \circ \bm f_{1} : \bm f_1,\cdots,\bm f_m \in \Fcal_{r,\bm \ell}, m \ge 1\},
    \end{equation}
    and
    \begin{equation}
    \eqrescnn_{r,\bm \ell} =\{ (\id + \bm f_{m}) \circ \cdots \circ (\id +\bm f_{1}) : \bm f_1,\cdots,\bm f_m \in \Fcal_{r,\bm \ell}, m \ge 1\}.
    \end{equation}

    To distinguish from $\mathbb R$ valued functions $\varphi : \mathbb X \to \mathbb R$,
    we use the word ``mappings'' to refer to functions from
    $\mathbb X$ to $\mathbb X$.
    \begin{definition}
    The mapping $\bm \varphi:\mathbb X \to \mathbb X$ is called shift equivariant if
    $$\trans_{\bm k} (\bm \varphi(\bm x)) = \bm \varphi(\trans_{\bm k}(\bm x)), \quad \forall \bm x \in \mathbb X,\,\forall \bm k.$$
    The mapping family $\mathcal X$ is said to have the shift equivariant UAP if
    \begin{enumerate}
    \item each mapping in $\mathcal X$ is shift equivariant, and
    \item given any shift equivariant continuous mapping $\bm \varphi$,
    compact set $K \subseteq \mathbb X$, and tolerance $\varepsilon > 0$,
    there exists a mapping $\bm \psi \in \mathcal X$ such that
    $$ \|\bm \psi -\bm \varphi \|_{L^p(K)} \le \varepsilon.$$
    \end{enumerate}
    \end{definition}
    Then, the analogous result with respect
    to equivariant approximation is stated as follows.
    
    \begin{theorem}
        \label{thm:eq}
        We have the following results.
        \begin{enumerate}[label=(\arabic*),ref=\ref{thm:eq}.(\arabic*)]
        \item \label{thm:eq1} For the fully convolutional neural network with residual blocks,
        it holds that $\eqrescnn_{r,\bm \ell}$ possesses the shift equivariant UAP for $r \ge 1$, and $\bm \ell \ge \bm 2$.
        For non-residual versions, $\eqcnn_{r,\bm \ell}$ possesses the shift equivariant UAP for $r \ge 2$ and $\bm \ell \ge \bm 2$.
        \item \label{thm:eq2} The kernel size $\bm 2$ is optimal in the following sense: for $\bm \ell$ with $\min \ell_{s}=1$, then neither $\eqrescnn_{\infty, \bm \ell}+\mathbb R$ nor $\eqcnn_{\infty,\bm \ell}+ \mathbb R$ possesses eq-UAP.
        \item \label{thm:eq3} The number of channel for non-residual fully convolutional neural network is optimal,
        in the sense that the mapping family $\eqcnn_{1,\bm \infty}+\mathbb R$ does not possess the shift equivariant UAP.
        \end{enumerate}
    \end{theorem}

    To prove \Cref{thm:main,thm:eq}, we start with the following proposition,
    which links the universal approximation property of invariant function family and that of an equivariant mapping family.
    A version of this was proved in \cite{li2022}
    in a rather abstract setting for general transitive groups.
    We provide a more explicit proof in the specific case where
    we are only concerned with shift operator $\trans_{\bm k}$.
    \begin{proposition}
    [Connection between Shift Invariant and Equivariant UAP]
    \label{prop:eq-in-suff} Suppose $g:\mathbb X \to \mathbb R$ is Lipschitz,
    permutation invariant,
    and $g(\mathbb X) = \mathbb R$. If a mapping family $\Acal$ possesses the shift equivariant UAP,
    then $\Bcal = \{ g\circ \bm \varphi : \bm \varphi \in \Acal\}$ possesses the shift invariant UAP.
    \end{proposition}

\begin{proof}
    Without loss of generality, we assume that $K = [-a,a]^{\bm n}$,
    otherwise we can enlarge $K$.
    Define $$K_1 = \{\bm x \in K : [\bm x]_{\bm 1} > [\bm x]_{\bm i}, \forall i \neq 1\}$$ as a subset of $K$.
    Then, it is easy to check that
    $K = \bigcup_{\bm i} (\trans_{\bm i} K_{1})$ up to a measure zero set.
    Define $\varepsilon' :=   \frac{\varepsilon}{|\bm n|(1+\lip g)}$,
    by results in \cite[Theorem 3.8]{li2019deep},
    for any $\varepsilon' > 0$ there exists $\bm{u}$ such that
    \begin{equation}
        \|F - g \circ \bm{u} \|_{L^p(K)} \le \varepsilon'.
    \end{equation}
    Note that $\bm{u}$ here is not necessarily equivariant, otherwise we are done.

    Now we attempt to find $\bm f$ by some kind of equivariantization on $\bm u$ as explained below.
    Since $\bm u $ is in $L^p$,
    we consider a compact set $O \subset K_1$
    such that $\|\bm{u}\|_{L^p(K_1 \setminus O)} \le \varepsilon'$.
    Take a smooth truncation function $\chi \in C^{\infty}(\R^d)$,
    whose value is in $[0,1]$,
    such that $\chi|_O = 1$ and $\chi|_{K_1^c} = 0$. Then $\tilde{\bm{u}} = \chi\bm{u}$ is a smoothly truncated version of $\bm u$.

    For $\bm x \in \trans_{\bm k} K_1$ with some index $\bm k$,
    define $\bm{f}(\bm{x}) = \trans_{\bm k} (\tilde{\bm{u}}(\trans_{-\bm k}(\bm x))).$
    Since different $\trans_{\bm k}K_1$ are disjoint,
    the value of $\bm{f}$ is unique in the union $\cup_{\bm k} \trans_{\bm k} K_1$.
    We set $\bm{f}(\bm x) = 0$ in the complement of $\cup_{\bm k} \trans_{\bm k} K_1$.
    The truncation function $\chi$ ensures that $\bm{f}$ vanishes on the boundary of $K_1$, therefore $\bm f$ is continuous,
    and direct verification shows that $\bm{f}$ is shift equivariant.

   It remains to estimate $\|F - g \circ \bm f\|_{L^p}$, since both $F$ and $g \circ \bm f$ are equivariant, it is natural and helpful to restrict our estimation on $K_1$, since
    \begin{equation}\|F - g \circ \bm f\|_{L^p(K)} = |\bm n| ~ \|F - g \circ \bm f \|_{L^p(K_1)}.\end{equation}
To estimate the error on $K_1$, we first bound the term $\|\bm u - \bm f\|_{L^p(K_1)}$. Since $\bm u $ and $\bm f|_{K_1} = \tilde {\bm u}$ coincide on $O$, we have

    \begin{equation}
        \begin{split}
            \|\bm{u} - \bm{f}\|_{L^p(K_1)} & = \|\bm u - \tilde{\bm u}\|_{L^p(K_1)} \\
            & \le \| \bm u \|_{L^p(K_1\setminus O)} = \varepsilon'.
        \end{split}
    \end{equation}
The inequality follows from the fact that $\chi$ takes value in $[0,1]$.
Since $g$ is Lipschitz,
we have $\|g\circ \bm{u} - g\circ \bm{f}\|_{L^p(K_1)} \le \lip g \varepsilon'$, yielding that $\|F - g\circ\bm{f}\|_{L^p(K_1)} \le (1+\lip g) \varepsilon' $. We finally have $\|F - g\circ\bm{f}\|_{L^p(K)} \le (1+\lip g)|\bm n| \varepsilon' = \varepsilon$.
\end{proof}

\begin{remark}
    \label{rmk:suff}
By \Cref{prop:eq-in-suff}, the first part of \Cref{thm:eq} immediately implies that of \Cref{thm:main}. Conversely, the second and third parts of \Cref{thm:main} almost imply those of \Cref{thm:eq}, if the added bias $+\mathbb R$ is omitted. To get the desired sharpness result, we will prove a more general function/mapping class that does not hold UAP, see \Cref{sec:sharp-channel-number} for details.
\end{remark}

\subsection{The dynamical systems approach}
To prove the first part of \Cref{thm:main}, we develop the dynamical systems approach
to analyze the approximation theory of compositional architectures
first introduced in \cite{li2019deep} without symmetry considerations, and subsequently extended to 
handle symmetric functions with respect to 
transitive subgroups of the permutation group \cite{li2022}.
While shift symmetry is covered under this setting,
the results in \cite{li2022} can only handle the case
where the convolution filters have the same size as
the input dimension.

In contrast, the results here are established for
small and constant filter (and channel) sizes.
This is an important distinction, as such configurations
are precisely those used in most practical applications.
On the technical side, the filter size restriction
requires developing new arguments to show how arbitrary
point sets can be transported under a flow - a key
ingredient in the proof of universal approximation
through composition (See \Cref{sec:proof} for a detailed discussion).
Furthermore, the restriction on filter sizes
also enabled us to address new questions,
such as a minimal size requirement, that cannot
be handled by the analysis in \cite{li2022}.
The results and mathematical techniques for
these sharpness results are new.
Concretely, to provide a sharp lower bound on
the filter size and channel number requirements, 
we develop some techniques to extract special features of functions in $\cnn_{1,\cdot} + \mathbb R$ and $\cnn_{\cdot, \bm 1} + \mathbb R$ that leads to the failure of universal approximation. Detailed constructions are found in \Cref{sec:sharp-kernel-size} and \Cref{sec:sharp-channel-number}. The construction and the corresponding analysis in this part are nontrivial, and we believe that the examples are also useful in analyzing the approximation property of other architectures.

The core technique we employ to analyze both $\cnn$ and $\rescnn$ is
the dynamical systems approach: in which we idealize residual networks
into continuous-time dynamical systems.
In this subsection, we introduce the key elements of this approach.

We first introduce the {\it flow map}, also called the {\it Poincar\'{e} mapping},
for time-homogenous dynamical systems.
\begin{definition}[Flow Map]
    Suppose $\bm f:\mathbb X \to \mathbb X$ is Lipschitz,
    we define the flow map associated with $\bm f$ at time horizon $T$
    as $\bm \phi(\bm f, T)(\bm x) = \bm z(T)$, where $\dot{\bm z}(t) = \bm f(\bm z (t))$ with initial data ${\bm z}(0) = \bm x$.
   
\end{definition}
 It follows from \cite{Arnold1973Ordinary} that the mapping
    $\bm \phi(f,T)$ is Lipschitz for any real number $T$,
    and the inverse of $\bm \phi(\bm f,T)$ is {$\bm \phi(-\bm f,T)$},
    hence the flow map is bi-Lipschitz.
    
Based on the flow map, we define the dynamical hypothesis space for the convolutional neural network. Define the dynamical hypothesis space with convolutional kernel as
\begin{equation}\begin{split}
\code_{r, \bm \ell} = \{ g \circ \bm \phi(\bm f_m, t_m)\circ &\cdots \circ \bm \phi(\bm f_1, t_1) :  \\ & \bm f_1,\cdots,\bm f_m \in \Fcal_{r,\bm \ell}, t_1,\cdots, t_m \in \mathbb R\},
\end{split}
\end{equation}
and the corresponding equivariant version as
\begin{equation}
    \begin{split}
\eqcode_{r,\bm \ell} =\{ \bm \phi(\bm f_{m},t_m) \circ &\cdots \circ \bm \phi(\bm f_{1},t_1) : \\ & \bm f_1,\cdots,\bm f_m \in \Fcal_{r,\bm \ell}, t_1,t_2,\cdots, t_m \in \mathbb R, m \ge 1\}.
    \end{split}
\end{equation}

The following proposition shows we can use residual blocks to approximate continuous dynamical systems.
\begin{proposition}
    \label{prop:time-disc}
    Suppose that $\Fcal$ is a bi-Lipschitz function family.
For given
$$\bm \Phi = \bm \phi(\bm f_m, t_m)\circ\cdots \circ \bm \phi(\bm f_1, t_1), \quad \bm f_i \in \Fcal,$$
and compact $K \subset \mathbb X$, $\varepsilon >0$, there exists
    $$
        \widehat{\bm \Phi} = (\id + s_{m'}\bm g_{m'}) \circ \cdots \circ (\id + s_1\bm g_{1}),
        \quad
        \bm g_i \in \Fcal
    $$
    for some $s_i > 0$, $i=1,\dots,m'$,
such that
$\|\bm \Phi - \widehat{\bm \Phi}\|_{L^p(K)} \le \varepsilon.$
\end{proposition}
\begin{proof}
See \cite[Section 3.2]{li2022}.
\end{proof}
The following result shows the shift invariant UAP for the continuous hypothesis spaces, which is the core part of this paper.

\begin{theorem}
\label{thm:code}
The dynamical hypothesis space $\code_{1, \bm 2}$ satisfies the shift invariant UAP, and $\eqcode_{1,\bm 2}$ satisfies the shift equivariant UAP.
\end{theorem}

Again, by \Cref{prop:eq-in-suff}, to prove \Cref{thm:code} it suffices to show that $\eqcode_{1, \bm 2}$ satisfies the shift equivariant UAP. This is rather technical, and we will spend the whole \Cref{sec:proof} to prove this theorem.

The rough proof strategy is as follows.
We reduce the problem to finite point transportation,
i.e., we need to show that  the hypothesis space can transport arbitrary but
finitely many points (in different orbits under the action of the translation group)
to any other set of points.
This is done in the previous work of \cite{li2022}, under a less restrictive setting.
A key technical difficulty here is that the kernel size is limited,
thus previous known constructions of point transportation (\cite{li2022})
cannot achieve this.
Here, we show that we can employ more composition of layers to construct
auxiliary mappings to achieve this transportation property.
The intuition is that finite-size kernels (satisfying some minimal requirements),
when stacked many times,
is as good as a full-sized kernel for domain rearrangement - a key
enabler of universal approximation through composition.

With this theorem in hand, we now prove the first part of Theorem~\ref{thm:main} and of Theorem~\ref{thm:eq}. 

\begin{proof}[Proof of the first part of Theorem~\ref{thm:eq}]
By the straightforward inclusion relationship, it suffices to show that the function family $\rescnn_{1,\bm 2}/\eqrescnn_{1,\bm 2}$ and $\cnn_{2,\bm 2}/\eqcnn_{2,\bm 2}$ have the corresponding UAP. For the residual version, it follows from \Cref{prop:time-disc} and \Cref{thm:code} that if $\code_{1,\bm 2}$ satisfies UAP,
then so does $\rescnn_{1,\bm 2}$. Similar argument holds for the pair $\eqcode_{1,\bm 2}$ and $\eqrescnn_{1, \bm 2}$. In other words, a convergent time discretization inherits universal approximation properties.
Thus, given Theorem~\ref{thm:code} it suffices to prove the remaining $\cnn$ and $\eqcnn$ case. In view of \Cref{rmk:suff}, it suffices to show the equivariant case.

We begin with a weaker result, showing that $\eqcnn_{3,\bm 2}$ satisfies UAP.
We prove that $\eqrescnn_{1,\bm 2} \subset \eqcnn_{3,\bm 2}$.
For given $\bm f = v\sigma(\bm w \ast \cdot + b \bm 1)$ with $\sigma = \relu$, we write
$$\bm x  + \bm f(\bm x) = \sigma(\bm x) + (-1) \sigma(-\bm x) + v\sigma(\bm w \ast \bm x + b \bm 1).$$
This relation indicates that $\eqrescnn_{1,\bm 2} \subset \eqcnn_{3,\bm 2}$, which means that $\eqcnn_{3,\bm 2}$ has UAP.

However, this approach cannot handle the case $\eqcnn_{2,\bm 2}$,
since the inclusion $\eqrescnn_{1,\bm 2} \subset \eqcnn_{2,\bm 2}$ does not hold.
This leads to a further modification in the following lemma, which completes the proof. 
\end{proof}
\begin{lemma}\label{lem:eqcnn22}
    For a given $\bm G \in \eqrescnn_{1,\bm 2}$,
    compact $K \subset \mathbb X$, there exists $\bm H \in \eqcnn_{2,\bm 2}$ such that
    $ \bm H(\bm x) = \bm G(\bm x)$ for all $\bm x \in K.$
\end{lemma}
\begin{proof}
In the following, we suppose that
$G = \bm f_M \circ\cdots\circ\bm f_1,$
where $\bm f_i(\bm x) = \bm x + v^i \sigma(\bm w^i \ast \bm x + b^i\bm 1).$
Set $\bm \gamma_{i} = \bm f_{i}\circ\cdots \bm f_1,\quad  \bm \gamma_{0} = \id$ and note that each $\bm \gamma_{i}$ is a Lipschitz mapping. We now consider a sufficiently large real number $R > 0$ such that $|\bm \gamma_i(\bm x)| \le R$ holds for all $i= 0,1,\cdots,M$ and $\bm x \in K$. This can be done since each $\gamma_{i}$ is Lipschitz, and $K$ is a compact set.
Define
\begin{equation}\label{eq:u0}\bm u_0(\bm x) = \sigma(\bm x + R\bm 1),\end{equation} and \begin{equation}\label{eq:ui}\bm u_i(\bm x) = \sigma(\bm x) + v^i\sigma(\bm w^i \ast \bm x + (b^i - (\Sigma_{\bm k}[\bm w^i]_{\bm k})R)\bm 1) \in \Fcal_{2,\bm 2} \end{equation}
for $i = 1,2,\cdots,M$.
Consider their composition
$
    \bm \eta_s(\bm x) = \bm u_s\circ \cdots \circ \bm u_1 \circ \bm u_0.
$
Clearly, $\bm \eta_s \in \cnn_{2, \bm 2}$ for all $s\geq 0$.
We now prove by induction that
\begin{equation}\label{eq:induction}
\bm\eta_i(\bm x) = \bm\gamma_i(\bm x) + R\bm 1 \text{ for }i = 0,1,\cdots,M.
\end{equation}
The base case ($i =0$) is obvious from the definition \eqref{eq:u0}, since $\bm x + R\bm 1 > \bm 0$ for all $\bm x \in K$.
Suppose that \eqref{eq:induction} holds for $i$, then
\begin{equation*}
\begin{split}\bm \eta_{i+1}(\bm x) = &\sigma(\bm \gamma_{i}(\bm x) + R\bm 1) + v^i\sigma(\bm w^i\ast (\bm \gamma_1(\bm x) +R\bm 1) + (b^i - (\Sigma_{\bm k}[\bm w^i]_{\bm k}) R) \bm 1)\\
 = & \bm \gamma_{i}(\bm x) + R\bm 1 + v^i\sigma(\bm w^i \ast \bm \gamma_i(\bm x) + b^i \bm 1) \\
 = & \bm \gamma_{i+1}(\bm x) + R\bm 1.
\end{split}
\end{equation*}
The first line uses the definition of \eqref{eq:ui},
and the second line follows from $\bm \gamma_i(\bm x) + R\bm 1 \ge 0$
and $\bm w^i \ast R\bm 1 = (\Sigma_{\bm k}[\bm w^i]_{\bm k}) R\bm 1$.
This proves \eqref{eq:induction} by induction.
Finally, we set
$
    \bm u_{M+1}(\bm x) = \sigma(\bm x) - \sigma(R\bm 1),$ then $\bm H(\bm x)
    : = (\bm u_{M+1}(\bm \eta_{M}(\bm x)))  = (\bm \gamma_{M}(\bm x)) = \bm G(\bm x)
$
for all $\bm x \in K$.
By construction, we have $\bm H \in \eqcnn_{2,\bm 2}$,
therefore we have proved that the UAP holds for $\eqcnn_{2,\bm 2}$.
\end{proof}
\begin{remark}
    We remark that the shift equivariance of the dynamical system
    (and the resulting flow map) may prompt one to consider the same equation
    in the quotient space with respect to shift symmetry, see \cite{cohen2016group}.
    However, in the case of flow approximation, we found no new useful tools
    in the quotient space to analyze approximation, thus this abstraction is not adopted here.

    We now give a concrete examples to show that we cannot directly deduce UAP from earlier results by a quotient argument. Observe that for the
non-symmetric setting, the result in \cite{li2019deep} requires that the
control family $\mathcal F$ be (restricted) affine invariant. If we directly
require this affine invariance in the quotient space, then it will be reduced to scaling
invariant. However, the scaling invariant property cannot induce the UAP, and the proof of this is similar to those in \Cref{sec:sharp-kernel-size}.
\end{remark}

\section{Sufficiency Results}
\label{sec:proof}



In this section we prove \Cref{thm:code}, i.e. the UAP of $\eqcode_{1,\bm 2}$.
Here, we relax the constraint that $\sigma = \relu$.
Instead, we make the following assumption on $\sigma$, which is called ``well function'' in \cite{li2019deep}.
\begin{definition}[Well Function]
We say a Lipschitz function $h : \mathbb R \to \mathbb R$ is a well function if $\{x \in \mathbb R : h(x) = 0\}$ is a bounded (closed) interval.
\end{definition}
In this section, we assume that there exists a well function $h$ in the closure of $\Span\{v\sigma(w\cdot+b), v,w,b \in \mathbb R\}$.
The commonly used activation functions meet this assumption, including ReLU, Sigmoid and Tanh, see \cite{li2019deep}. 

Before the main part of this section, let us first introduce some additional definitions.
\begin{definition}[Coordinate Zooming Function]
For a given continuous function $u : \mathbb R \to \mathbb R$, define the coordinate zooming function
$u^{\otimes} : \mathbb X \to \mathbb X$ by
$$[u^{\otimes}(\bm x)]_{\bm i} = u([\bm x]_{\bm i}).$$
\end{definition}
\begin{definition}[Stabilizer]
We say a point $\bm x$ is a stabilizer if and only if there exists a non-trivial $\bm k \neq 0$, such that $\trans_{\bm k} \bm x=\bm x.$
\end{definition}
\begin{definition}[Shift Distinct]
We say a point set $X = \{\bm x^1, \bm x^2, \dots, \bm x^n\}$
is shift distinct, if for some $i_1,i_2,\bm k$ with
$\trans_{\bm k}(\bm x^{i_1}) = \bm x^{i_2}$, 
then we must have $i_1 = i_2$ and $\bm k = \bm 0$.
\end{definition}

Notice that if a point set $X$ is shift distinct,
then for any member $\bm x \in X$, the only $\bm k$ so that $\trans_{\bm k}(\bm x) = \bm x$ is $\bm k = \bm 0$. This is implied by the definition of shift distinctness.

The proof of \Cref{thm:code} is based on the following approximation framework, which relies on the following introduced two properties of a mapping family.

\begin{proposition}[Basic Framework]\label{prop:framework}
Given a family $\Acal$ of mappings $\mathbb X \to \mathbb X$, and
suppose $\Acal$ is closed under composition. If $\Acal$ satisfies the following two conditions:
\begin{enumerate}
\item[1.] (Coordinate zooming property) For any continuous function $u$, the mapping $u^{\otimes}$ is in $\Acal$.
\item[2.] (Point matching property) For a given shift distinct point set $\bm x^1,\cdots,\bm x^M$,
a target point set $\bm y^1,\cdots,\bm y^M$,
and a stabilizer point set
$\bm s^{1},\bm s^{2},\cdots,\bm s^{N}$, a tolerance $\varepsilon > 0$,
there exists a mapping $\bm \varphi \in \Acal$ such that
$$|\bm\varphi(\bm x^i) - \bm y^i|\le \varepsilon$$ and $$|\bm \varphi(\bm s^i)| \le 1.$$
\end{enumerate}
Then, $\Acal$ possesses the shift equivariant UAP.
\end{proposition}

Note that for the point matching property is to say, we can use mappings in $\Gcal$ to move each $\bm x^i$ to $\bm y^i$, while keeping a stabilizer set stay around the original point.
We now use this proposition to prove Theorem~\ref{thm:code}. Consider the closure in the UAP sense, that is,
\begin{equation}
    \begin{split}
\Acal :=  & \{ \bm \varphi : \text{ for all compact }K \subset \mathbb X, \varepsilon > 0, \text{ there exists } \bm \psi \in \eqcode_{1,\bm 2},  \\ &\|\bm \varphi - \bm \psi\|_{L^p(K)} \le \varepsilon\}.
    \end{split}
\end{equation}

Note that the non-equivariant version is the main object studied in \cite{li2019deep}. The following proposition collects some basic properties of $\Acal$, which serves as a toolbox when proving \Cref{thm:code}.
\begin{proposition}\label{prop:Acal}
The following results hold for the mapping family $\Acal$.
\begin{enumerate}
\item $\Acal$ is closed under composition.
\item Given $\bm w \in \mathbb X$, $\supp \bm w \le \bm 2$, and $b \in \mathbb R$, then the flow map $\bm \phi(h(\bm w \ast \cdot + b \bm 1), t) \in \Acal$.
\item $\Acal$ satisfies the coordinate zooming property.
\item If $\Acal$ possesses shift equivariant UAP, then so does $\eqcode_{1,\bm 2}$.
\end{enumerate}
\end{proposition}
\begin{proof}
See \cite[Section 3.3]{li2019}.
\end{proof}

\begin{proposition}
Suppose now the point matching property holds for $\Acal$, then \Cref{thm:code} holds.
\end{proposition}
\begin{proof}

By the last part of \Cref{prop:Acal}, it suffices to show that $\Acal$ possesses shift equivariant UAP. By the third part (and the first part), we know that if $\Acal$ has the point matching property, then $\Acal$ has shift equivariant UAP, which concludes the result.
\end{proof}

From the proof, we know that:
Once the point matching property is proved, Theorem~\ref{thm:code} is then proved.
The proof of the point matching property is the most
technical part in this paper.
We first give a sketch of the proof.

\begin{proof}[Sketch of the proof of the point matching property]
In this sketch, we only consider the case when there are no stabilizers,
i.e. when $N = 0$.
\begin{enumerate}%
\item[Step 1.] We first show that if $\Acal$ has the following point reordering property, then $\Acal$ has the point matching property.

\begin{quote}
    (Point Reordering Property) For any shift distinct point set $i$, we can find a mapping $\bm \varphi \in \Acal$ such that
    $$[\bm \varphi(\bm x^j)]_{\bm i} > [\bm \varphi(\bm x^{j'})]_{\bm i'}$$ if $j < j'$ or $j = j'$ but $\bm i \prec \bm j$. Here the partial order $\prec$ is the lexicographic order.
    For brevity, we say in this case that $\bm\varphi(\bm x^j)$ is ordered.
\end{quote}
\item[Step 2.] To begin with, we first prove that there exists a mapping $\bm \beta \in \Acal$, such that
$$[\bm \alpha(\bm x^1)]_{\bm i} > [\bm \alpha(\bm x^j)]_{\bm i'}$$
for $j \neq 1$ and any indices $\bm i, \bm i'$.

\item[Step 3.] Set $\bm z^j = \bm \alpha(\bm x^j)$. Now we are ready for an induction argument.
Suppose for $\bm z^2,\cdots, \bm z^M$ we have a mapping $\bm \psi \in \Acal$ to fulfill the point reordering property.
We modify it to the mapping $\tilde{\bm \psi} = \bm \psi\circ u^{\otimes} \in \Acal$, such that it satisfies the following conditions
\begin{itemize}
\item[-] $\tilde{\bm \psi}(\bm z^2),\cdots, \tilde{\bm \psi}(\bm z^M)$ are ordered.
\item[-] $[\tilde{\bm \psi}(\bm z^1)]_{\bm i} > [\tilde{\bm \psi}(\bm z^j)]_{\bm i'}$ for $j \neq 1$, and indices $\bm i$, $\bm i'$.
\end{itemize}
\item[Step 4.] Finally, we modify $\tilde{\bm \psi}$ to get $\bm \varphi$ such that
$\bm \varphi(\bm z^j)$ is ordered. Till now, we prove the point reordering property for $\Acal$.
\end{enumerate}
\end{proof}
The full proof of Theorem~\ref{thm:code} is put in \Cref{sec:proof:eqcode:complete}.

\subsection{Proof of \Cref{prop:framework}}

\begin{proof}[Proof of \Cref{prop:framework}]
Without loss of generality,
we can suppose that $K = [-a,a]^{\bm n}$.
Otherwise, we can expand $K$ to a sufficiently large hypercube.

\paragraph{\textbf{Step 1.}}

Given a scale $\delta > 0$, consider the grid $\delta\mathbb Z^{\bm n}$ with size $\delta$.
Let $\bm q \in \mathbb Z^n$ be a tensor with all coordinates being integers,
and $\chi_{\bm q}$ be the indicator of the cube
\begin{equation}
    \square_{\bm q,\delta} := \big\{ \bm x : [\bm x]_{\bm i} \in [[\bm q]_{\bm i}\delta, ([\bm q]_{\bm i} + 1) \delta ] \big\}.
\end{equation}

Since $\bm{\varphi}$ is in $L^p(K)$,
by standard approximation theory $\bm{\varphi}$
can be approximated by equivariant piecewise constant (and shift equivariant) functions
\begin{equation}
    \bm{\varphi}_0(\bm x) = \sum_{\bm{q}, \square_{\bm q,\delta} \subset K} \bm{y}_{\bm q} \chi_{\bm q}(\bm x),
\end{equation}
where
\begin{equation}
    \bm y_{\bm q}
    =
    \bm\lambda(\square_{\bm q,\delta})^{-1} \int_{\square_{\bm q,\delta}} \bm \varphi(\bm x) d\bm x
\end{equation}
is the local average value of $\bm \varphi$ in $\square_{\bm q,\delta}$.
Then, we have
\begin{equation}
    \|\bm \varphi - \bm \varphi_0\|_{L^p(K)} \le \omega_{\bm{\varphi}}(\delta)
    [\bm\lambda(K)]^{1/p} \to 0\end{equation} as $\delta \to 0
$,
where
$\omega_{\bm \varphi}$ is the modulus of continuity
(restricted to the region $K$),
i.e.,
\begin{equation}
    \omega_{\bm \varphi}(\delta)  := \sup_{|\bm x - \bm y| \le \delta} |\bm \varphi(\bm x) - \bm \varphi(\bm y)|
\end{equation}
for $\bm x$ and $\bm y$ in $K$
and $\bm\lambda(K)$ is the Lebesgue measure of $K$.

\paragraph{\textbf{Step 2.}}
Let $\bm q \delta$ be a vertex of $\square_{\bm q,\delta}$.
Define $\Ical$ as the maximal subset of
$\Ical_0 = \{ \bm q:\bm q\delta \in K\}$
such that $\{ \bm q\delta : \bm q \in \Ical\}$ is shift distinct.
By the maximal property,
and the definition of shift distinctness, for each $\bm q \in \Ical_0$, only two situations can happen:
\begin{enumerate}
\item there exists a shift operator $\trans_{\bm k}$ and $\bm q' \in \Ical$,
such that $\trans_{\bm k} {\bm q'} = \bm q,$ or
\item $\bm q$ itself is a stabilizer, that is, there exists a shift operator $\trans_{\bm k}$ with $\bm k \neq \bm 0$ such that $\trans_{\bm k} \bm q = \bm q.$
\end{enumerate}
By the construction of $\bm y_{\bm q}$, it holds that $$\trans_{\bm k} \bm y_{\bm q} = \bm y_{\trans_{\bm k}\bm q},\quad\forall \bm k, \forall \bm q.$$
Given $\varepsilon>0$,
by the point matching property, we can find $\bm f \in \Acal$ such that
\begin{itemize}
    \item[-] for $\bm q \in \Ical_0$ that is not a stabilizer,
    $|\bm f(\bm q \delta) - \bm y_{ \bm q}| \le \varepsilon$;
    \item[-] for $\bm q \in \Ical_0$ that is a stabilizer,
    $|\bm f(\bm q\delta)| \le 1$.
\end{itemize}

For $\alpha \in (0,1)$, define the shrunken cube
\begin{equation}
    \square_{\bm q,\delta}^{\alpha}
    :=
    \{\bm x \in \mathbb X : [\bm x]_{\bm i} \in [[\bm q]_{\bm i}\delta, ([\bm q]_{\bm i} + \alpha)\delta]\},
\end{equation}
and define $K^{\alpha} = \bigcup_{\square_{\bm q,\delta} \subset K} \square_{\bm q,\delta}^{\alpha}$,
which is a subset of $K$.
Given $\beta > 0$, we now use the coordinate
zooming property of $\Acal$
to find $u^{\otimes} \in \Acal$
such that

\begin{equation}
    \label{eq:cond-on-u}
    u([ih,(i+\alpha h)]) \subset [ih,(i+\frac{\beta}{n}\delta)]
    \text{ for }i \in \{ i_s : s=1,\dots,n; \bm i \in \Ical\}.
\end{equation}
To do this, we construct a piecewise linear function
$u$ such that
    \begin{equation}
         u|_{[i\delta, (i+\alpha) \delta]}(x) = i+\frac{\beta}{2n} \delta,
    \end{equation}
    by setting
    \begin{equation}
         u|_{[(i+\alpha) \delta,
        (i+1)\delta]}(x) = (x - (i-\alpha)\delta)/(1 - \alpha)
        + i + \frac{\beta}{2n} \delta
    \end{equation}
    explicitly,
and select $\varepsilon < \frac{\beta}{3n} \delta$.
By the coordinate zooming property, it holds that $u^{\otimes} \in \Acal$.

Therefore, we have

\begin{equation}
    \label{eq:esti-p-general-position}
    |\bm{f}(u^{\otimes}(\bm x)) - \bm y_{\bm{q}}|
    \le 2\varepsilon \text{ for }\bm x \in \square_{\bm q,\delta}^{\alpha}
,\end{equation}
if $\bm q$ is not a stabilizer, and
\begin{equation}
    \label{eq:esti-p-not-general-position}
    |\bm f(u^{\otimes}(\bm x))| \le 1+ \varepsilon \text{ for }\bm x \in \square_{\bm q,\delta}^{\alpha}
,\end{equation}
if $\bm q$ is a stabilizer.

These two estimates \eqref{eq:esti-p-general-position} and \eqref{eq:esti-p-not-general-position} will be useful in the final step.

\paragraph{\textbf{Step 3.}}
We are now ready to estimate the error $\|\bm \varphi - \bm f \circ u^{\otimes}\|_{L^p(K)}$.
The estimation is split into three parts,
\begin{equation} \begin{split}K \setminus K^{\alpha}, \\  K^{\alpha}_1 = \bigcup_{\bm q\text{ is not a stabilizer}} \square_{\bm q,\delta}^{\alpha},  \\ K^{\alpha}_2 = \bigcup_{\bm q\text{ is a stabilizer}} \square_{\bm q,\delta}^{\alpha}.\end{split}
\end{equation}
Notice that $K^{\alpha} = \bigcup \square_{\bm q,\delta}^{\alpha}$.

For $K^{\alpha}_1$, from \eqref{eq:esti-p-general-position} in the end of Step 2, we have $\|\bm f \circ u^{\otimes} - \bm \varphi_0\|_{L^{\infty}(K_1^{\alpha})} \le 2\varepsilon$, and thus
\begin{equation}
    \label{eq:K1alpha-estimate}
    \|\bm f \circ u^{\otimes} - {\bm \varphi}_0\|_{L^p(K^{\alpha}_1)}
    \le
    2\varepsilon[\bm \lambda(K^{\alpha})]^{1/p} \le 2\varepsilon[\bm \lambda(K)]^{1/p}.
\end{equation}

For $K^{\alpha}_2$, note that if $\bm q$ is a stabilizer, then all points in $\square_{\bm q,\delta}$ will be close to a hyperplane $$\Gamma_{\bm i, \bm j}  := \{\bm x \in \mathbb X : [\bm x]_{\bm i} = [\bm x]_{\bm j} \}$$ for some distinct $\bm i,\bm j$, the distance from those points to $\Gamma_{\bm i,\bm j}$ will be smaller than $\sqrt{|\bm n|}\delta$.
Therefore, the Lebesgue measure of $K_2^{\alpha} \subset K_2$ will be smaller than that of all points whose distance to the union of hyperplanes $\Gamma_{\bm i, \bm j}$ is less than $\sqrt{n}\delta$, which is $O(\delta)$.
Thus, we have
\begin{equation}
    \label{eq:K2alpha-estimate}
    \begin{split}
    \|\bm f \circ u^{\otimes} - {\bm \varphi}_0\|_{L^p(K^{\alpha}_2)} \le~& (1+ {\varepsilon}+ \|\bm \varphi_0\|_{C(K)})O(\delta) \\ \le ~& (1+ {\varepsilon}+ \|\bm \varphi\|_{C(K)})O(\delta).
    \end{split}
\end{equation}
The last line holds since $\|\bm \varphi_0\|_{C(K)} \le \|\bm\varphi\|_{C(K)}$ by construction.

For $K \setminus K^{\alpha}$, we have
\begin{equation}
    \label{eq:Kalpha-estimate}
    \begin{split}
    \|\bm f \circ u^{\otimes} - {\bm \varphi}_0\|_{L^p(K \setminus K^{\alpha})} &
    \le (\|\bm f\|_{C(K)}+\|\bm \varphi\|_{C(K)})~\bm \lambda(K \setminus K^{\alpha})^{1/p}\\ & \le (\|\bm f\|_{C(K)}+\|\bm \varphi\|_{C(K)})(1-\alpha^d)^{1/p}[\bm \lambda(K)]^{1/p}.
    \end{split}
\end{equation}

We first choose $\delta$ sufficiently small such that the right hand side of \eqref{eq:K2alpha-estimate} is not greater than $\varepsilon$, then choose $\alpha$ such that  $1- \alpha$ is sufficiently small,
and $(\|\bm f\|_{C(K)}+\|\bm \varphi\|_{C(K)})(1-\alpha^d)^{1/p} \le \varepsilon$. The we conclude the result since $\bm f \circ u^{\otimes} \in \Acal$.

 \end{proof}

\subsection{Complete Proof of Theorem~\ref{thm:code}}
\label{sec:proof:eqcode:complete}
In this section, we complete the proof of Theorem~\ref{thm:code}.
As discussed at the beginning of this section,
we first consider the case when there are no stabilizers to be dealt with.

~

\paragraph{\textbf{Step 1.}}

We first show that if $\Acal$ has the following point reordering property, then $\Acal$ has the point matching property.

\begin{quote}
    (Point Reordering Property) For any shift distinct point set $\bm x^j, j = 1,2,\cdots, M$,
    we can find a mapping $\bm \varphi \in \Acal$ such that
    $$[\bm \varphi(\bm x^j)]_{\bm i} > [\bm \varphi(\bm x^{j'})]_{\bm i'}$$ if $j < j'$ or $j = j'$ but $\bm i \prec \bm j$. Here the partial order $\prec$ is the lexicographical order. For brevity, we say in this case $\bm\varphi(\bm x^j)$ is \textit{ordered}.
\end{quote}

Without loss of generality we can assume that $\bm y^j$ is also shift distinct. Suppose there exist $\bm \varphi_{\bm x}$ and $\bm \varphi_{\bm y} \in \Acal$ such that $\bm \varphi_{\bm x}(\bm x^j), j = 1,2,3,\cdots, M$ is ordered, and $\bm \varphi_{\bm y}(\bm y^j), j = 1,2,\cdots, M$ is ordered. Then we can find a continuous mapping $u$ such that
$$ u([\bm \varphi_{\bm x}(\bm x^j)]_{\bm i}) = u([\bm \varphi_{\bm y}(\bm y^j)]_{\bm i})$$
holds for $j = 1,2,\cdots,M$ and all the indices $\bm i$.
Therefore, the mapping $\bm \varphi_{\bm y}^{-1} \circ u^{\otimes} \circ \bm \varphi_{\bm x}$ is then constructed to satisfy the point matching property.

~

\paragraph{\textbf{Step 2.}}

To begin the proof of the point matching property,
we first prove that there exists a mapping $\bm \beta \in \Acal$, such that
$$[\bm \alpha(\bm x^1)]_{\bm i} > [\bm \alpha(\bm x^j)]_{\bm i'}$$
for $j \neq 1$ and any indices $\bm i, \bm i'$.
We first show that we can perturb the point set
$\bm x^j, j = 1,2,\cdots, M $ such that all the coordinate $[\bm x^j]_{\bm i}, j = 1,2,\cdots, M, \bm i \le \bm n,$ are different.
In what follows, we say that in this case $\bm x^j$ are \textit{perturbed}.

The perturbation argument is based on the following minimal argument.
For $\bm \alpha \in \Acal$, consider the following quantity:
$$E(\bm \alpha) = \{ (\bm i,j,\bm i',j) : \bm i \neq \bm i' \text{ or } j \neq j', [\bm \alpha(\bm x^j)]_{\bm i} = [\bm \alpha(\bm x^{j'})]_{\bm i} \}.$$
Suppose $\bm \alpha$ minimizes this quantity, it suffices to show that $E(\bm \alpha) = 0$. Otherwise, we consider a pair $(\bm I, J)$ and $(\bm I', J')$ sch that $(\bm I, J) \neq (\bm I', J')$ but
$$[\bm \alpha(\bm x^J)]_{\bm I}  = [\bm \alpha(\bm x^{J'})]_{\bm I'}.$$
Since $\bm \alpha(\bm x^J)$ and $\bm \alpha(\bm x^{J'})$ must be shift distinct, no matter whether $J$ and $J'$ are identical, we can deduce that there exists a $\bm k$ and $\bm e$, where $\bm e = (0,0,\cdots, 1,\cdots, 0,\cdots, 0)$, such that
\begin{equation}
[\bm \alpha(\bm x^J)]_{\bm I+\bm k} = [\bm \alpha(\bm x^{J'})]_{\bm I' + \bm k}
\end{equation}
but
\begin{equation}
    [\bm \alpha(\bm x^J)]_{\bm I+\bm k+ \bm e } \neq [\bm \alpha(\bm x^{J'})]_{\bm I' + \bm k + \bm e }.
\end{equation}
So without loss of generality,
we may assume that $\bm k = \bm 0$, and $\bm e = (1,0,0,\cdots,0)$.

Consider the following dynamics
$$\frac{d}{dt}[\bm z]_{\bm i} = f(\bm z) = \sigma([\bm z]_{\bm i+\bm e}+b).$$
Here the constant $b$ is chosen to ensure that the
$$    \sigma([\bm \alpha(\bm x^J)]_{\bm I+\bm e } +b) = 0 \neq \sigma([\bm \alpha(\bm x^{J'})]_{\bm I'  + \bm e } + b).$$
Then for sufficiently small $t > 0$, the inequality
$$E(\bm \phi(\bm f,t)\circ u^{\otimes} \circ \bm \alpha) < E(\bm \alpha)$$
leads to a contradiction of minimality.
Therefore, there exists $\bm \alpha \in \Acal$ such that $\bm \alpha(\bm x^j),j = 1,2,\cdots, M$ is perturbed.

~

\paragraph{\textbf{Step 3.}}
So far, we can assume that $\{\bm x^j, j = 1,2,\cdots, M\}$ itself is perturbed since we can apply a perturbation $\bm \alpha$ constructed in the previous step to achieve it otherwise.

Consider the following quantity:
$$K(\bm \alpha) = \{(\bm i, \bm i',j) : j \neq 1, [\bm \alpha(\bm x^j)]_{\bm i} \ge  [\bm \alpha (\bm x^1)]_{\bm i'}\}.$$
We choose an $\bm \alpha \in \Acal$ to minimize this quantity in $\Acal$ subject to $\bm \alpha (\bm x^j)$ is perturbed. Now it suffices to prove that $K(\bm \alpha) = 0$. Suppose not, then there exists $(\bm I, J, \bm I')$ such that $$\omega = [\bm \alpha(\bm x^J)]_{\bm I} - [\bm \alpha(\bm x^1)]_{\bm I'}$$ is the smallest one among all choice that makes the above value non-negative. Clearly, by the minimality of $(\bm I, J, \bm I')$, no other $[\bm \alpha (\bm x^j)]_{\bm i}$ is inside the interval $([\bm \alpha(\bm x^1)]_{\bm I'}, [\bm \alpha(\bm x^J)]_{\bm I})$. Since we have assumed $\bm \alpha(\bm x^j), j = 1,2,\cdots, M$ is perturbed, then $\omega\neq0$.

We define a continuous function $v : \mathbb R \to \mathbb R$,  such that
\begin{itemize}
\item[-] $v([\bm \alpha(\bm x^J)]_{\bm I}) - v([\bm \alpha(\bm x^1)]_{\bm I'}) \le \mu$;
\item[-] for other pairs, $(\bm i,j,\bm i') \neq (\bm I, J, \bm I')$, it holds that $$ |v([\bm \alpha(\bm x^j)]_{\bm i}) - v([\bm \alpha(\bm x^{j'})]_{\bm i'})| \ge 3\varepsilon$$ for all $j$.
\end{itemize}

Here $\varepsilon$ and $\mu$ are two parameters whose values
will be determined later.
Simply speaking, what we did is just squeezing the coordinates and make $v([\bm \alpha(\bm x^J)]_{\bm I})$ and $v([\bm \alpha(\bm x^1)]_{\bm I'})$ very close to each other. The following picture illustrates this, we use the coordinate zooming function to squeeze the coordinates.

\begin{picture}(0,0)\setlength{\unitlength}{0.5cm}
    \put(0,0){\line(1,0){10}}
    \multiput(1,0)(1,0){3}{{\color{red} \circle*{0.2}}}
    \put(4,0){{\color{blue} \circle*{0.2}}}
    \put(5,0){{\color{orange} \circle*{0.2}}}
    \multiput(6,0)(1,0){3}{{\color{green} \circle*{0.2}}}
    \put(11,0){$\Rightarrow$}
    \put(13,0){\line(1,0){10}}
    \multiput(13.3,0)(1.5,0){3}{{\color{red} \circle*{0.2}}}
    \put(17,0){{\color{blue} \circle*{0.2}}}
    \put(17.3,0){{\color{orange} \circle*{0.2}}}
    \multiput(18,0)(1.5,0){3}{{\color{green} \circle*{0.2}}}

   \end{picture}

   ~

Set $\bm I_1 = \bm I + (1,0,\cdots,0)$ and $\bm I'_{1} = \bm I' + (1,0,\cdots,0)$. Consider the dynamics (for short, we only write the equations {for the coordinates} we are concerned with)
\begin{equation}
\begin{cases}
\frac{d}{dt} [\bm z]_{\bm I} = v\sigma([\bm z]_{\bm I_1}+b), \\ 
\frac{d}{dt}[\bm z]_{\bm I'} = v\sigma([\bm z]_{\bm I'_1}+b).
\end{cases}
\end{equation}

We choose certain $b$ and $v \in \mathbb R$ such that
$$v\sigma([\bm \alpha(\bm x^J)]_{\bm I} + b + \varepsilon) = 0$$and
$$v\sigma([\bm \alpha(\bm x^{1})]_{\bm I'} + b - \varepsilon) = 1.$$

From the classical ODE theory, the dynamics will move each $[\bm \alpha(\bm x^{j})]_{\bm i}$ to $[\bm a(t)^{j}]_{\bm i}$ such that
$$|\bm \alpha(\bm x^{j})_{\bm i} - [\bm a(t)^{j}]_{\bm i}| \le C_1(e^{tC_2} - 1),$$
for some constants $C_1, C_2$ depending only on $\varepsilon$.
We choose a sufficiently small $t > 0$ such that the right hand side of the above inequality is less than $\varepsilon$.
Therefore, we always have
$$v\sigma([\bm a(t)^J]_{\bm I} + b) = 0$$ and
$$v\sigma([\bm a(t)^1]_{\bm I'} + b)  > \delta := \min_{r \in[1+\varepsilon, 1+3\varepsilon]} v\sigma(r).$$
Note that this $\delta$ only depends on $\varepsilon$.
Then we can choose $\mu = \min(\frac{1}{2}t\delta, \varepsilon^2)$, and therefore we at least have
$$ [\bm a(t)^J]_{\bm I} > [\bm a(t)^1]_{\bm I'},$$
while the other order are preserved since $\varepsilon$ is now much larger than $\mu$. This contradicts with the minimal choice of $\bm \alpha$. Hence, we conclude the result.

The above argument just constructs a special dynamics
to switch two coordinates that are squeezed.
Concretely, we exchange the position of the points in orange and blue shown below.

\begin{picture}(0,0)\setlength{\unitlength}{0.5cm}
    \put(0,0){\line(1,0){10}}
    \multiput(0.3,0)(1.5,0){3}{{\color{red} \circle*{0.2}}}
    \put(4,0){{\color{blue} \circle*{0.2}}}
    \put(3.5,0.3){{\color{blue} \vector(1,0){1}}}
    \put(4.3,0){{\color{orange} \circle*{0.2}}}
    \put(4.8,-0.3){{\color{orange} \vector(-1,0){1}}}
    \multiput(5,0)(1.5,0){3}{{\color{green} \circle*{0.2}}}
    \put(11,0){$\Rightarrow$}
    \put(13,0){\line(1,0){10}}
    \multiput(13.3,0)(1.5,0){3}{{\color{red} \circle*{0.2}}}
    \put(17,0){{\color{orange} \circle*{0.2}}}
    \put(17.3,0){{\color{blue} \circle*{0.2}}}
    \multiput(18,0)(1.5,0){3}{{\color{green} \circle*{0.2}}}

   \end{picture}

~

\paragraph{\textbf{Step 4.}}

Set $\bm z^j = \bm \alpha(\bm x^j)$. 
Now we are ready to proceed with induction.
Suppose for $\bm z^2,\cdots, \bm z^M$ we have a mapping $\bm \psi \in \Acal$ to fulfill the point reordering property. We modify it to the mapping $\tilde{\bm \psi} = \bm \psi\circ u^{\otimes} \in \Acal$, such that it satisfies the following conditions
\begin{itemize}
\item[-] $\tilde{\bm \psi}(\bm z^2),\cdots, \tilde{\bm \psi}(\bm z^M)$ are ordered.
\item[-] $[\tilde{\bm \psi}(\bm z^1)]_{\bm i} > [\tilde{\bm \psi}(\bm z^j)]_{\bm i'}$ for $j \neq 1$, and indices $\bm i$, $\bm i'$.
\end{itemize}

The idea is simple: Since in the previous step the data $\bm z^j$ has been transformed so
that $[\bm z^1]_{\bm i}$ is the largest among all coordinates,
it suffices to construct a dynamics to drive all coordinates $[\bm z^1]_{\bm i}$ away from the other coordinates,
so as to leave enough space for performing the induction.
In what follows, the points in red represents the coordinates $[\bm z^1]_{\bm i}$,
while the points in black represents the others. 

~

\begin{picture}(0,0)\setlength{\unitlength}{0.5cm}
    \multiput(0.2,-1)(0,0.2){10}{\circle*{0.1}}
    \put(0,0){\line(1,0){10}}
    \multiput(0.5,0)(0.5,0){10}{\circle*{0.2}}
    \put(6,0){\color{red}\circle*{0.2}}
    \put(7,0){\color{red}\circle*{0.2}}
    \multiput(8,-1)(0,0.2){10}{\circle*{0.1}}
    \put(1.5,-1){induction region}

    \put(11,0){$\Rightarrow$}
    \put(13,0){\line(1,0){10}}
    \multiput(13.2,-1)(0,0.2){10}{\circle*{0.1}}
    \put(0,0){\line(1,0){10}}
    \multiput(13.5,0)(0.5,0){10}{\circle*{0.2}}
    \put(21.5,0){\color{red}\circle*{0.2}}
    \put(22,0){\color{red}\circle*{0.2}}
    \multiput(21,-1)(0,0.2){10}{\circle*{0.1}}
    \put(14.5,-1){induction region} 

\end{picture}

\hspace{10em}

By restricting $\bm \psi$ in the line $\mathbb R \bm 1$, we can obtain a continuous increasing bijection from $\mathbb R \bm 1$ to $\mathbb R \bm 1$. As a result, we can find $a > 0$, such that $$a > \max_{j \neq 1} \max_{\bm i} [\bm z^j]_{\bm i} + 2$$ and $$\bm \gamma(a\bm 1) > \max_{\bm i} [\psi(\bm z^j)]_{\bm i} + 2 \lip \bm \psi.$$
Define $u$ such that $u$ fixes all $[\bm z^j]_{\bm i}$ for $j \neq 1$ and all indices $\bm i$, but sends $[\bm z^1]_{\bm i}$ to the interval $[a-1,a+1]$.
We consider $\tilde{\bm \psi} = \bm \psi\circ u^{\otimes} \in \Acal$,
which satisfies the following conditions
\begin{itemize}
\item[-] $\tilde{\bm \psi}(\bm z^2),\cdots, \tilde{\bm \psi}(\bm z^M)$ are ordered;
\item[-] $[\tilde{\bm \psi}(\bm z^1)]_{\bm i} > [\tilde{\bm \psi}(\bm z^j)]_{\bm i'}$ for $j \neq 1$, and indices $\bm i$, $\bm i'$.
\end{itemize}

\paragraph{\textbf{Step 5.}}

Set $\bm p^j = \tilde{\bm \psi}(\bm z^j)$.
 Consider the following quantity for $\gamma \in \Acal$ such that
\begin{enumerate}
\item[-] $\bm \gamma (\bm p^2),\cdots, \bm \gamma(\bm p^M)$ is ordered;
\item[-] $[\bm \gamma(\bm p^1)]_{\bm i} > [\bm \gamma(\bm p^j)]_{\bm i'}$ for $j \neq 1$, and indices $\bm i = \bm i'$;
\end{enumerate}
we define the following quantity

$$L(\bm \gamma) = \{ (\bm i, \bm i') : \bm i \prec \bm i', [\bm \gamma(\bm p^1)]_{\bm i} < [\bm \gamma(\bm p^1)]_{\bm i'}\}.$$
We claim that this quantity can achieve zero.
Suppose not, then we can find a pair $(\bm I, \bm I')$ such that
$$\bm I \prec \bm I'$$ but $$ [\bm \gamma(\bm p^1)]_{\bm I'} - [\bm \gamma(\bm p^1)]_{\bm I}$$ is minimal among all the choices that make this value positive.
One can verify  that there must be no other indices $\bm k$ such that
$$  [\bm \gamma(\bm p^1)]_{\bm I'} >  [\bm \gamma(\bm p^1)]_{\bm k} >  [\bm \gamma(\bm p^1)]_{\bm I},$$
since in this case either $\bm I \prec \bm k$ or $\bm k \prec \bm I'$ should be satisfied, which contradicts with the choice of $\bm I$ and $\bm I'$.
Therefore, there exists a continuous function $v$ such that
$$v([\bm \gamma(\bm p^1)])_{\bm I'} - v([\bm \gamma(\bm p^1)]_{\bm I}) \le \mu,
$$
and
$$|v([\bm \gamma(\bm p^j)]_{\bm i}) - v([\bm \gamma(\bm p^{j'})]_{\bm i'})| \ge \varepsilon.$$
From a similar argument (squeeze and switch) as in Step 2, we can construct a new $\tilde{\bm \gamma}$ satisfying the condition but with minimal $L(\tilde{\bm \gamma})$. Therefore, we conclude the result.

~

\paragraph{\textbf{Dealing with Stabilizers}}
We conclude the proof with the situation where there are
stabilizers. 
The motivation behind the following argument is to push the coordinates of all the stabilizers (which are in red) to the leftmost area.

\begin{picture}(0,0)\setlength{\unitlength}{0.5cm}
    \put(0,0){\line(1,0){10}}
    \put(1,0){\circle*{0.2}}
    \put(2,0){\circle*{0.2}}
    \put(3,0){\color{red} \circle*{0.2}}
    \put(4,0){ \circle*{0.2}}
    \put(5,0){\color{red} \circle*{0.2}}
    \put(6,0){ \circle*{0.2}}
    \put(7,0){ \circle*{0.2}}
    \put(8,0){ \color{red}\circle*{0.2}}
    \put(11,0){$\Rightarrow$}
    \put(13,0){\line(1,0){10}}
    \put(14,0){\color{red} \circle*{0.2}}
    \put(15,0){\color{red} \circle*{0.2}}
    \put(16,0){\color{red} \circle*{0.2}}
    \put(17,0){ \circle*{0.2}}
    \put(18,0){ \circle*{0.2}}
    \put(19,0){ \circle*{0.2}}
    \put(20,0){ \circle*{0.2}}
    \put(21,0){ \circle*{0.2}}

   \end{picture}

We prove that there exists a $\bm \zeta \in \Acal$ such that
\begin{equation}
[\bm \zeta (\bm x^j)]_{\bm i} > [\bm \zeta(\bm z^{j'})]_{\bm i'}
\end{equation}
for all the possible choices of $(\bm i, j), (\bm i',j')$.
In such a $\bm \zeta$ exists, we can proceed,
as we did in Step 3, to find a $\bm \varphi \in \Acal$, such that
\begin{enumerate}
\item  $$[\bm \varphi(\bm x^j)]_{\bm i} > [\bm \varphi(\bm x^{j'})]_{\bm i'}$$ if $j < j'$ or $j = j'$ but $\bm i \prec \bm j$.
\item $$[\bm \varphi (\bm x^j)]_{\bm i} > [\bm \varphi(\bm z^{j'})]_{\bm i'}$$ for all possible choice of $(\bm i, j), (\bm i',j')$.
\end{enumerate}
We first show that, with this point reordering property with stabilizers, we can prove the point matching property.
Compared to what we did in Step 1, it suffices to additionally assign the value $\bm o^{1},\bm o^{2},\dots,\bm o^{N}$, as the target of $\bm s^{j}$, where each coordinate of $\bm p^j$ is chosen to be $\le 1$.

Suppose for target $\bm y^{1},\cdots, \bm y^{M}$ and $\tilde{\bm o} = 0$, such a $\varphi_{\bm y} \in \Acal$ can be found. We choose $\bm o^{1},\cdots, \bm o^N$ around the value $\varphi_{\bm y}(\tilde{\bm o})$, such that
$$|\bm o^{i} - \varphi_{\bm y}(\tilde{\bm o})| \le \varepsilon,$$
and moreover we can assign $v$ in these value such that
$v([\bm s^j]_{\bm i}) = [\bm o^j]_{\bm i}$. Hence, the requirement of point matching property can be fulfilled if we choose $\varepsilon < \frac{1}{2}(\lip \bm \varphi_{\bm y}^{-1})^{-1}.$

Now we prove the existence of $\bm \zeta$. Consider the following quantity
$$H(\bm \zeta) := \{(\bm i, j, \bm i', j'): [\bm \zeta (\bm x^j)]_{\bm i} \le [\bm \zeta(\bm s^{j'})]_{\bm i'},
\}$$
and we choose $\bm \zeta \in \Acal$ to minimize this quantity. We only need to show that $H(\bm \zeta) = 0$.
Otherwise, we can prove that we can construct
a new $\bm \zeta \in \Acal$ with a lower value of $H(\bm \zeta)$.

This construction is similar to what we did in Step 2, in that
we only need to find a such a pair $(\bm I, J,\bm I', J')$ such that
\begin{equation}
    \label{eq:zeta} [\bm \zeta(\bm s^{J'})]_{\bm I'} \ge  [\bm \zeta (\bm x^J)]_{\bm I}
\end{equation}
and there are no other coordinates between these two value,
but with an $\bm e = (0,\cdots,1, \cdots,0)$ such that
$[\bm \zeta(\bm x^J)]_{\bm I+\bm e} \neq  [\bm \zeta(\bm s^{J'})]_{\bm I'+\bm e}$. 

To complete the proof, 
we assert that such pairs can be found.
Suppose this assertion does not hold. Since we assume that $H(\bm \zeta) \neq 0$, which immediately implies that there exists at least one pair $J$ and $J'$, such that for some multi-indices $\bm I$ and $\bm I'$, \eqref{eq:zeta} holds.
We choose such $(\bm I, J, \bm I', J')$ to minimize the quantity
$$\omega = [\bm \zeta(\bm s^{J'})]_{\bm I'}-  [\bm \zeta (\bm x^J)]_{\bm I}.$$
If this quantity $\omega$ does not equal to zero, then clearly there is no other coordinates between these two values. But since the assertion
does not hold, we can derive that
$$[\bm \zeta(\bm x^J)]_{\bm I+\bm e} =  [\bm \zeta(\bm s^{J'})]_{\bm I'+\bm e}.$$
Therefore, the quantity $\omega$ should be zero.
Thus, the problem reduces the case when $\omega = 0$.

In this case, we start from a pair $[\bm x^J]_{\bm I} = [\bm s^{J'}]_{\bm I'}$, we can show that for all $\bm e = (0,\cdots,1,\cdots,0)$, we have
$$[\bm x^J]_{\bm I+\bm e} = [\bm s^{J'}]_{\bm I'+\bm e}.$$
Repeating this procedure, we can know that the above identity holds for all choice of $\bm e$. Therefore, there exists a shift operator
$\trans_{\bm k}$ such that $\trans_{\bm k}\bm x^{J} = \bm s^{J'}$,
which also leads to a contradiction, since it implies that $\bm x^J$ is a stabilizer.

\section{Sharpness Results}
This section proves the sharpness result, i.e., the second and third part of \Cref{thm:main,thm:eq}. The proof in this section are all constructive.
\subsection{Sharpness of the kernel size requirement}
\label{sec:sharp-kernel-size}
In this subsection, we prove the second part of Theorem~\ref{thm:main}.
Consider the kernels with support $\bm \ell$ such that $\min \ell_s = 1$.
Without loss of generality we can assume that $\ell_1 = 1$.

We use the following example to illustrate the main intuition behind this sharpness result.
More precisely, we show that the sum of two univariate function cannot approximate a bivariate function well.
As an explicit example, we show that there exists $\varepsilon_0$ such that
$$\|xy - f(x) - g(y)\|_{L^p([0,1]^2)} \ge \varepsilon_0$$
for all choice of $L^p$ functions $f$ and $g$.
Suppose that for some $f,g \in L^p([0,1]^2)$,
$\|xy - f(x) - g(y)\|_{L^p([0,1]^2)} =  \varepsilon.$ we define
$I = [0,1/2]^2$ and $\bm p_1 = [0,0], \bm p_2 = [1/2,0],\bm p_3 = [0,1/2], \bm p_4 = [1/2,1/2]$.
For convenience, denote by $h(x,y) = xy$.
Consider the following value
$$M = \|h(\bm x + \bm p_1) + h(\bm x + \bm p_4) - h(\bm x+\bm p_2) - h(\bm x + \bm p_3)\|_{L^p(I)}.$$
Direct calculation yields that $M = 4^{-\frac{p+1}{p}}>0.$ However, for $\widehat{h}(x,y) = f(x) + g(y)$, it holds that
$$\widehat h(\bm x + \bm p_1) + \widehat h(\bm x + \bm p_4) - \widehat h(\bm x+\bm p_2) - \widehat h(\bm x + \bm p_3) = 0.$$
By triangle inequality,
$M \le 0 + \sum_{i = 1}^{4} \| h (\bm x + \bm p_i) - \hat{h}(\bm x+\bm p_i)\|_{L^p(K)} \le 4 \varepsilon.$
Therefore, it holds that $\varepsilon\ge  \frac{M}{4} > 0$, concluding the result.

For the general case of establishing the sharpness result, we mimic the example above.
We introduce the following auxiliary space.
For $\bm x \in \mathbb X$ and integer $I$, define $\bm x_{I:}$ as the tensor in
$\mathbb X_1 = \mathbb R^{n_2\times\cdots \times n_d}$,
such that
$[\bm x_{I:}]_{(i_2,\cdots, i_d)} = [\bm x]_{(I,i_2,\cdots,i_d)}.$
Define $\Hcal$ as the mapping $\mathbb X \to \mathbb R$ such that
$$\Hcal := \{g \circ \bm \varphi: \mathbb X \to \mathbb R: \exists \bm \psi :\mathbb X_1 \to \mathbb X_1, \text{ such that } [\bm \varphi(\bm x)]_{I:} = \bm \psi(\bm x_{I:}) \}.$$
We illustrate the function family $\Hcal$ in the following $\mathbb R^{3\times 3}$ example.
If $F \in \Hcal$, then $F$ should have the following form:
$$F = g\left( \begin{bmatrix} \bm \psi(\bm x_{(1,1)},\bm x_{(1,2)}, \bm x_{(1,3)}) \\ \bm \psi(\bm x_{(2,1)},\bm x_{(2,2)}, \bm x_{(2,3)}) \\ \bm \psi(\bm x_{(3,1)},\bm x_{(3,2)}, \bm x_{(3,3)}) \end{bmatrix} \right).$$
By the assumption on $\bm \ell$, it is straightforward to deduce that
$\cnn_{r,\bm \ell} + \mathbb R$, $\rescnn_{r,\bm \ell} + \mathbb R$, $\code_{r,\bm \ell} + \mathbb R$ are all in $\Hcal$.
It remains to show that $\Hcal$ does not possess the shift invariant UAP.
The idea follows the simple example above,
by noting that in this case $x$ is now
$(\bm x_{(1,1)},\bm x_{(1,2)}, \bm x_{(1,3)})$,
$y$ is now $(\bm x_{(2,1)},\bm x_{(2,2)}, \bm x_{(2,3)})$,
and $f(x) + g(y)$ is now some general permutation invariant function.
We now carry out this proof.
\begin{proof}[Proof of the second part of Theorem~\ref{thm:main}]
As discussed before, it suffices to show that $\Hcal + \mathbb R$ does not satisfy UAP.
Let us set
$
    F(\bm x) = \prod_{i_1 > i_2} (\psi([\bm x]_{i_1:}) - \psi([\bm x]_{i_2:})),$ where $ \psi(\bm y) = \prod_{\bm i'} [\bm y]_{\bm i'},
$
and $K = [0,1]^{\bm n}$, we show that there exists a constant $\varepsilon_0 > 0$, such that for all $H \in \Hcal$, it holds
\begin{equation}
\label{eq:Hcal}
\|F - H\|_{L^p(K)} \ge \varepsilon_0.
\end{equation}
Choose two subregions of $K$,
$K_1 = \{ \bm x \in K, \bm x_{1:} \gg \bm x_{2,:} \gg \cdots \gg\bm x_{n_1:}\},$
and
$K_2 = \{ \bm x \in K, \bm x_{2:} \gg \bm x_{1:} \gg x_{3:} \gg \cdots \gg\bm x_{n_1:}\}.$
Here, we say for $\bm z_1$ and $\bm z_2 \in \mathbb X_1$,
$\bm z_1 \gg \bm z_2$ means $\min_{\bm i}[\bm z_1]_{\bm i} \ge \max_{\bm i}[\bm z_2]_{\bm i}.$
Consider the mapping $\bm \tau$, that flips first and second rows (along first index), that is,
\begin{equation}
[\bm \tau(\bm x)]_{2:} = [\bm x]_{1:}, \quad [\bm \tau(\bm x)]_{1:} = [\bm x]_{2:}, \quad [\bm \tau(\bm x)]_{i:} =  [\bm x]_{i:}, i \neq 1,2.
\end{equation}
Then $\bm \tau(K_1) = K_2$.
By the definition of $\Hcal$, we have
$(H\circ \tau)(\bm x) = H(\bm x)$ for $\bm x \in K_1$,
and $H \in \Hcal + \mathbb R$. But $F\circ \tau= -F$, which implies that
\begin{equation}
    \begin{split}
2\|F\|_{L^p(K_1)} &= \|F - F\circ\bm \tau\|_{L^p(K_1)} \\ & \le \|H - H\circ\bm \tau\|_{L^p(K_1)} + \|F - H\|_{L^p(K)} + \|F\circ \bm \tau - H\circ \bm \tau\|_{L^p(K)} \\ & = 2\|F - H\|_{L^p(K)}.
    \end{split}
\end{equation}
In the last equation, the last two terms are equal since $\bm \tau$ is measure preserving.
\end{proof}

For the equivariant version, we propose $$\Hcal_{eq} := \{ \bm \varphi: \mathbb X \to \mathbb R: \exists \bm \psi :\mathbb X_1 \to \mathbb X_1, \text{ such that } [\bm \varphi(\bm x)]_{I:} = \bm \psi(\bm x_{I:}) \},$$ which contains $\eqcnn_{r,\bm \ell} + \mathbb R, \eqrescnn_{r,\bm \ell} + \mathbb R$ and $\eqcode_{r,\bm \ell} + \mathbb R$. The corresponding proof is quite similar.

\begin{proof}[Proof of the second part of Theorem~\ref{thm:eq}]
Set $[\bm F(\bm x)]_{\bm j} = \prod_{i_1 > i_2} (\psi([\bm x]_{i_1:}) - \psi([\bm x]_{i_2:}))$ for each index $\bm j$, where $ \psi(\bm y) = \prod_{\bm i'} [\bm y]_{\bm i'}$. The remaining argument is similar to the previous one.
\end{proof}
\subsection{Sharpness of the channel number requirement}
\label{sec:sharp-channel-number}
In this subsection, we show that the FCNN with only one channel per layer cannot satisfy the shift invariant UAP.
The key to proving this part is the following observation.
Suppose $G \in \cnn_{1,\bm \infty} + \mathbb R$, then
$G$ is continuous, piecewise linear. Moreover, by direct calculation, we obtain that there exists $\bm g \in \mathbb X$, such that for a.e. $\bm x \in K$, the gradient of $G$ is $\bm 0$ or $\bm g$.
The last assertion can be proved from direct calculation on the gradient of $G$.
\begin{proof}[Proof of the third part of Theorem~\ref{thm:main}]
Based on the above observation, we now show that $F(\bm x) = |\bm x|$ cannot be approximated by such $G$ in the unit ball $B(\bm 0,1)$.
By a change of variables we rewrite
\begin{equation}
    \int_{\bm x \in B(0,1)} |F(\bm x) - G(\bm x)|^p d \bm x
    =
    \int_{\bm \xi \in \partial B(0,1)}\int_0^1 |F(t\bm \xi) - G(t\bm \xi)|^p t^{|\bm n |- 1} dt dS,
\end{equation}
where $|\bm n| = n_1n_2\cdots n_d$.
We consider the hemisphere defined by $\bm \xi \in \partial B(0, 1)$
such that $\bm \xi \cdot \bm g < 0$.
On this hemisphere, $f(t) = F(t\bm \xi) = t$ is increasing while $g(t) = G(t\bm \xi)$ is decreasing
in $t$.

To proceed, we state and prove the following lemma.
\begin{lemma}
For $f : [a,b] \to \mathbb R$ that is increasing, we have
$$\inf_{g \text{ decreasing in } [a,b]} \int |f - g|^p = \inf_{g \text{ constant in} [a,b]}\int |f - g|^p.$$
\end{lemma}
\begin{proof}
The $\leq$ part is obvious, so it suffices to prove the $\geq$ part.
Given any decreasing $g$, set a constant $\tilde g$ such that
$\tilde g(t) = g(t_0)$ if $f(t_0) = g(t_0)$ for some $t_0$, and $g = f(1)$ if there does not exist such a $t_0$.
We can easily verify that $|f(t) - g(t)| \ge |f(t) - \tilde g(t)|$ for all $ t\in [0,1]$.
\end{proof}
Using this lemma, we can show that
\begin{equation}
    \begin{split}
\int_0^1|f(t) - g(t)|^pt^{|\bm n| -1} dt
 \ge &
 \int_{1/2}^1|f(t) - g(t)|^pdt \cdot (\frac{1}{2})^{|\bm n|-1}
 \\ \ge &
 (\frac{1}{2})^{|\bm n|-1}\inf_{a \in [1/2,1]} \int_{1/2}^1|f(t) - a|^pdt
 \\ = &
 (\frac{1}{2})^{|\bm n|-1}\inf_{a \in [1/2,1]} \frac{(1-a)^{p+1} + (a - 1/2)^{p+1}}{p+1}
 \\ = & 2^{-|\bm n| + 1}\cdot 2 \cdot \frac{(1/4)^{p+1}}{p+1} =: C_p.
    \end{split}
\end{equation}
The last line follows from the fact that the
minimization problem attains its infimum at $a = \frac{3}{4}.$
Therefore,
$\int_{\bm x \in B(0,1)} |F(\bm x) - G(\bm x)|^p > \frac{C_p}{2}\alpha$,
where $\alpha$ is the Lebesgue measure of $\partial B(0,1)$.
This implies the third part of Theorem~\ref{thm:main}.
\end{proof}

For the equivariant version, we can prove the third part of Theorem~\ref{thm:eq} in a similar way.

\section{Conclusion}
\label{sec:conclusion}
In this paper,
we provided the first approximation result of deep fully convolutional neural networks
with the fixed channel number and limited convolution kernel size,
and quantify the minimal requirements on these to achieve universal approximation
of shift invariant (or equivariant) functions.
We proved that the fully convolutional neural network with residual blocks
$\rescnn_{r,\bm \ell}$ achieves shift invariant UAP if and only if $\bm r \ge 1$ and
$\bm \ell \ge \bm 2$. This result does not require the specific form of the
activation function. For the non-residual version, we proved that $\cnn_{r, \bm
\ell}$ has the shift invariant UAP if and only if $\bm r \ge 2$ and $\bm \ell
\ge \bm 2$. The if part requires specifying $\sigma$ to be the ReLU operator. In
addition, the results also hold for their corresponding equivariant versions.
The proof is based on developing tools for dynamical hypothesis spaces, which have
the flexibility to handle variable architectures and obtain approximation results
that highlight the power of function composition.

\revise{

}
We conclude with some discussion on future directions.
In this paper, the shift invariant UAP for $\cnn_{2,\bm 2}$ was established for ReLU
activations.
The proof relies on the special structure of ReLU: $\relu(x) = x$ for $x > 0$,
hence we can make use of translation to replace the residual part.
This construction was outlined in the proof of the first part of Theorem~\ref{thm:main}.
It will be of interest to study if the other activations, such as sigmoid or tanh,
can also achieve shift-invariant UAP at fixed widths and limited kernel sizes.
Further, one may wish to establish explicit approximation rates
in terms of depth, and identify suitable function classes that can be efficiently approximated
by these invariance/equivariance preserving networks.
Finally, one may also consider extending the current theory to handle up-sampling and down-sampling
layers that are commonly featured in deep architectures.

In addition to approximation error, it is very natural and useful to consider the 
generalization error (statistical error) in the overall analysis of a machine learning model.
Compared to shallow and wide models, few generalization
results in the deep-but-narrow setting (for layers greater than
3) have been established.
While the current paper only concerns approximation theory,
it is nevertheless an important future direction to establish generalization estimates.

\bibliography{ref}
\bibliographystyle{plain}
\end{document}

%% file: preamble.tex
\usepackage[utf8]{inputenc} 
\usepackage{amsmath}
\usepackage{amsfonts}
\usepackage{amsthm}
\usepackage{hyperref}
\hypersetup{
    colorlinks=true,
    linkcolor=cyan,
    filecolor=blue,      
    urlcolor=red,
    citecolor=green,
}
\usepackage{cleveref}
\usepackage{picture}
\usepackage{graphicx} 

\usepackage{booktabs} 
\usepackage{array} 
\usepackage{paralist} 
\usepackage{verbatim} 
\usepackage{subfig} 

\usepackage{bm}

\usepackage{geometry}
\usepackage{geometry}
\usepackage{xcolor}

\usepackage{lineno}
\usepackage{amssymb}
\usepackage{mathrsfs}

\newcommand{\cnn}{\mathbf{CNN}}

\newcommand{\rescnn}{\mathbf{resCNN}}
\newcommand{\eqcnn}{\mathbf{eqCNN}}
\newcommand{\eqrescnn}{\mathbf{eqresCNN}}
\newcommand{\eqcode}{\mathbf{eqCODE}}
\newcommand{\code}{\mathbf{CODE}}

\DeclareMathOperator*{\id}{id}

\newcommand{\R}{\mathbb{R}}

\newcommand{\Acal}{\mathcal{A}}
\newcommand{\Bcal}{\mathcal{B}}

\newcommand{\Fcal}{\mathcal{F}}
\newcommand{\Gcal}{\mathcal{G}}
\newcommand{\Hcal}{\mathcal{H}}
\newcommand{\Ical}{\mathcal{I}}

\newcommand{\revise}[1]{{\color{red}#1}}

\newcommand{\lip}{\mathrm{Lip}~}
\newcommand{\relu}{\mathrm{ReLU}}

\usepackage{color}

\newcommand{\trans}{{\mathcal T}}

\DeclareMathOperator{\Span}{span}
\DeclareMathOperator{\supp}{supp}